\documentclass[pmlr]{jmlr}

\makeatletter
\def\ps@firstpage{%
  \def\@oddhead{}%
  \def\@evenhead{}%
  \def\@oddfoot{\normalfont\hfil\thepage\hfil}%
  \def\@evenfoot{\normalfont\hfil\thepage\hfil}%
}
\makeatother

\usepackage{hyperref}
\usepackage{microtype}      
\usepackage{amsmath}
\usepackage{nicefrac}       
\usepackage{url}            

\usepackage{verbatim}
\usepackage{enumitem}
\usepackage{euscript}
\usepackage{latexsym}
\usepackage{graphics}
\usepackage{graphicx}

\usepackage{url}
\usepackage{xspace}

\usepackage{color}
\usepackage{wrapfig}
\usepackage[section]{algorithm}
\usepackage[noend]{algorithmic}

\usepackage{graphicx}

\newcommand{\hp}{\ensuremath{\hat{\phi}}}
\newcommand{\R}{\ensuremath{\mathbb{R}}}

\newcommand{\denselist}{\itemsep -2pt\parsep=-1pt\partopsep -2pt}

\newcommand{\omt}[1]{}

\newcommand{\Eu}[1]{\ensuremath{\EuScript{#1}}}
\newcommand{\eps}{\varepsilon}

\newcommand{\s}[1]{{\textsf{ \small #1}}}
\newcommand{\E}{\s{\textbf{E}}}
\renewcommand{\Pr}{\s{\textbf{Pr}}}

\jmlryear{(Preprint 2026)}

\title{Relative Error Embeddings of the Gaussian Kernel Distance}

\begin{document}

\author{%
\Name{Di Chen}\thanks{Contributed as a student at HKUST while visiting University of Utah, partially supported by RGC grants GRF-16208415, GRF-621413 and GRF-16211614.} 
  \Email{dchenad@connect.ust.hk}\\
  \addr Noah's Ark Lab, Huawei Technologies\\ 
  Units 525-530, Core Building 2\\
  Hong Kong Science Park, Shatin, Hong Kong 
  \AND 
  \Name{Jeff M. Phillips}\thanks{Thanks to support by NSF CCF-1350888, IIS-1251019, ACI-1443046, and CNS-1514520.} 
  \Email{jeffp@cs.utah.edu}\\ 
  \addr University of Utah\\ 
  50 S Central Campus Dr. 3190\\
  Salt Lake City, UT 84112\\
  United States of America
} 

\maketitle

\begin{abstract}
	A reproducing kernel defines an embedding of a data point into an infinite dimensional reproducing kernel Hilbert space (RKHS).  The norm in this space describes a distance, which we call the kernel distance.  The random Fourier features (of Rahimi and Recht) describe an oblivious approximate mapping into finite dimensional Euclidean space that behaves similar to the RKHS.  We show in this paper that for the Gaussian kernel the Euclidean norm between these mapped to features has $(1+\eps)$-relative error with respect to the kernel distance.    When there are $n$ data points, we show that $O((1/\eps^2) \log n)$ dimensions of the approximate feature space are sufficient and necessary.  
	Without a bound on $n$, but when the original points lie in $\R^d$ and have diameter bounded by $\mathcal{M}$, then we show that $O((d/\eps^2) \log \mathcal{M})$ dimensions are sufficient, and that this many are required, up to $\log(1/\eps)$ factors.  
	We empirically confirm that relative error is indeed preserved for kernel PCA using these approximate feature maps.  
\end{abstract}


\section{Introduction}

The kernel trick in machine learning allows for non-linear analysis of data using many techniques such as PCA and SVM which were originally designed for linear analysis.  The ``trick'' is that these procedures only access data through inner products between data points, and the standard dot product can be replaced with a non-linear inner product kernel $K(\cdot,\cdot)$.  
Now given $n$ data points, one can compute the $n \times n$ gram matrix $G$ of all pairwise inner products; that is so $G_{i,j} = K(x_i, x_j)$ for all $x_i, x_j$ in input data set $X$.  
Then the analysis can proceed using just the gram matrix $G$.  

However, for large data sets, constructing this $n \times n$ matrix is a computational bottleneck, so methods have been devised for lifting $n$ data points $P \subset \R^d$ to a high-dimensional space $\R^m$ (but where $m \ll n$) so that the Euclidean dot product in this space approximates the non-linear inner product defined by $K$.  

For reproducing kernels $K$, there exists a lifting $\phi : \R^d \to \Eu{H}_K$, where $\Eu{H}_K$ is the reproducing kernel Hilbert space.  It is in general infinite dimensional, but every finite subset of $n$ points $\Phi(X) = \{\phi(x) \mid x \in X\}$ has the span of an $n$-dimensional Euclidean space.  That is $K(x,y) = \langle \phi(x), \phi(y) \rangle$.  
Moreover, we can define the norm of a point in $\Eu{H}_K$ as $\|\phi(x)\|_{\Eu{H}_K} = \sqrt{\langle \phi(x), \phi(x) \rangle}$ using the inner product, and then due to linearity, a distance (the \emph{kernel distance}) between two points is defined:		
\begin{align*}
D_K(x,y) = \|\phi(x) - \phi(y)\|_{\Eu{H}_K}
=& 
\sqrt{\|\phi(x)\|_{\Eu{H}_K}^2 + \|\phi(y)\|_{\Eu{H}_K}^2 - 2 \langle \phi(x), \phi(y) \rangle}\\
=&
\sqrt{K(x,x) + K(y,y) - 2K(x,y)}.
\end{align*}
For reproducing kernels (actually a subset called characteristic kernels) this is a metric~\citep{SGFSL10,muller97}.  

Thus we may desire an approximate lifting $\hat \phi : \R^d \to \R^m$ such that with probability at least $1-\delta$ for all $x,y \in X$
\[
(1-\eps) \leq \frac{D_K(x,y)}{\|\hat \phi(x) - \hat \phi(y)\|} \leq (1+\eps).  
\]
It turns out, one can always algorithmically construct such a lifting with $m = O((1/\eps^2) \log (n/\delta))$ by the famous Johnson-Lindenstrauss (JL) Lemma~\citep{JL84}.  However, unlike the JL Lemma, there is not always known an implicit construction.  In general, we must first construct the $n \times n$ gram matrix, revealing an $n$-dimensional subspace (through an $O(n^3)$ time eigendecomposition) and then apply $m = O((1/\eps^2) \log (n/\delta))$ random projections.  

So in recent years there have been many types of kernels considered for these implicit embeddings with various sorts of error analysis, such as for Gaussian kernels~\citep{rahimi2007random,lopez2014randomized,SS15,SS15-UAI}
group invariant kernels~\citep{li2010random}, 
min/intersection kernels~\citep{maji2009max}, 
dot-product kernels~\citep{kar2012random}, 
information spaces~\citep{AKMVV16},
and polynomial kernels~\citep{hamid2013compact,ANW14}.  

In this document we reanalyze one of the most widely used and first variants, the Random Fourier Features, introduced by \cite{rahimi2007random}.  It applies to symmetric shift-invariant kernels which include Laplace, Cauchy, and most notably Gaussian.  We will primarily focus on  Gaussian kernels, defined $K(x,y) = e^{-\frac{\|x - y\|^2}{2 \sigma^2}}$, unless specified otherwise. 
It is characteristic, hence $D_K$ is a metric.  

\subsection{Existing Properties of Gaussian Kernel Embeddings}

\cite{rahimi2007random} defined two approximate embedding functions: 
$\tilde \phi : \R^d \to \R^m$ and $\hat \phi : \R^d \to \R^m$ (defined below).  Only the former appears in the final version of paper, but the latter is also commonly used throughout the literature~\citep{SS15-UAI}.  
Both features use random variables $\omega_i \in \R^d$ drawn uniformly at random from the Fourier transform of the kernel function; in the case of the Gaussian kernel, the Fourier transform is again a Gaussian, specifically $\omega_i \sim \mathcal{N}_d(0,\sigma^{-2})$.  

In the former case, they define $m$ functions of the form
$\tilde f_i(x) = \cos(\langle \omega_i, x\rangle + \gamma_i)$, where $\gamma_i \sim \textsf{Unif}(0,2\pi]$, uniformly at random from the interval $(0,2\pi]$, is a random shift.   
Applying each $\tilde f_i$ on a datapoint $x$ gives the $i$th coordinate of $\tilde \phi(x)$ in $\R^{m}$ as $\tilde \phi(x)_i = \tilde f_i(x) / \sqrt{m}$. 

In the latter case, they define $t = m/2$ functions of the form
\[
\hat f_i(x) = \left[\begin{array}{c}\cos(\langle \omega_i, x \rangle) \\ \sin(\langle \omega_i, x \rangle)\end{array}\right]
\]
as a single $2 \times 1$ dimensional vector, and one feature pair.  Then applying $\hat f_i$ on a data point $x$ yields the $(2i)$th and $(2i+1)$th coordinate of $\hat \phi(x)$ in $\R^m$ as $[\hat \phi(x)_{2i}; \hat \phi(x)_{2i+1}] = \hat f_i(x) / \sqrt{t}$.

\cite{rahimi2007random} showed $\E[\langle \tilde \phi(x),  \tilde \phi(y)\rangle] = K(x,y)$ for any $x,y \in \R^d$, and that this implied 
\[
\Pr[| \langle \tilde \phi(x), \tilde \phi(y) \rangle - K(x,y) | \geq \eps ] \leq \delta
\]
\begin{itemize} \denselist
	\item with $m = O((1/\eps^2) \log(1/\delta))$ for each $x,y \in \R^d$, 
	%
	\item with $m = O((1/\eps^2) \log(n/\delta))$, for all $x,y \in X$, for $X \subset \R^d$ of size $n$, 
	%
	or
	\item with $m = O((d/\eps^2) \log(\mathcal{M}/\delta))$, for all $x,y \in X$, for $X \subset \R^d$ so $\mathcal{M} = \displaystyle{\max_{x,y \in X}} \|x-y\|/\sigma$.  
\end{itemize}
Recently \cite{SS15} improved the constants in these bounds, and showed rate optimality.  It is folklore (apparently removed from final version of \citep{rahimi2007random}; reproved \citep{SS15-UAI}) that also  $\E[\langle \hat \phi(x), \hat \phi(y) \rangle] = K(x,y)$, and thus all of the above PAC bounds hold for $\hat \phi$ as well.  
\citep{SS15-UAI} also compared $\tilde \phi$ and $\hat \phi$ (they used $\breve \phi$ for our $\tilde \phi$ and $\tilde \phi$ for our $\hat \phi$), and demonstrated that $\hat \phi$ performs better (for the same $m$) and has provably lower variance in approximating $K(x,y)$ with $\langle \hat \phi(x), \hat \phi(y) \rangle$ as opposed to with $\langle \tilde \phi(x), \tilde \phi(y) \rangle$.  
However, these results do \emph{not} obtain a bound on $\|\hat \phi(x) - \hat \phi(y)\| / D_K(x,y)$ since for very small distances, the additive error bounds on $K(x,y)$ are not sufficient to say much about $D_K(x,y)$.

\subsection{Our Results}

In this paper we show that $\hat \phi$ probabilistically induces a kernel 
$\hat K(x,y) = \langle \hat \phi(x), \hat \phi(y)\rangle$ and a distance
\[
D_{\hat K}(x,y) = \sqrt{\|\hat \phi(x)\|^2 + \|\hat \phi(y)\|^2 - 2 \hat K(x,y)} = \|\hat \phi(x) - \hat \phi(y)\|,
\]
which has strong relative error bounds with respect to $D_K(x,y)$, namely for a parameter $\eps \in (0,1)$
\begin{equation} \label{eq:rel-err}
(1-\eps) \leq \frac{D_K(x,y)}{D_{\hat K}(x,y)} \leq (1+\eps).
\end{equation}

In Section \ref{sec:basic} we show (\ref{eq:rel-err}) holds for each $x,y$ such that $\|x-y\|/\sigma \geq 1$, with probability at least $1-\delta$, with $m = O((1/\eps^2) \log(1/\delta))$.  We also review basic properties about $\hat \phi$ and $D_K$.  

We first prove bounds that depend on the size $n$ of a data set $X \subset \R^d$.  
We show that $m = O((1/\eps^2) \log n)$ features are necessary (Section \ref{sec:local-l2}) and sufficient (Section \ref{sec:small-data}) to achieve (\ref{eq:rel-err}) with high probability (e.g., at least $1-1/n$), when $d$ and $X$ are otherwise unrestricted.  

In Section \ref{sec:lowD} we prove bounds for $X \subset \R^d$ where $d$ is small, but the size $n = |X|$ is unrestricted.  Let $\mathcal{M} = \max_{x,y \in X} \|x-y\|/\sigma$.  
We show that $m = O((d/\eps^2) \log(\frac{d}{\eps} \frac{\mathcal{M}}{\delta}))$ is sufficient to show (\ref{eq:rel-err}) with probability $1-\delta$.   Then in Section \ref{sec:LowDimLB} we show that $m = \Omega(\frac{d}{\eps^2 \log(1/\eps)} \log (\frac{\mathcal{M}}{\log(1/\eps)}))$ is necessary for any feature map.  

In Section \ref{sec:exp} we empirically confirm the relative error through simulations.  
This includes showing kernel PCA obtains relative error using these approximate features.  

\subsection{Implications in Machine Learning and Data Analysis}
\label{sec:impl}

These new relative error bounds have numerous implications in machine learning and geometric data analysis.  We mention a couple others involving geometric approximations in learning and mining, and in an $L_1$ bound on Gram matrix approximations in Section \ref{sec:discuss}.  

\paragraph{Limits on oblivious kernel embeddings.}
There has been extensive recent effort to find oblivious subspace embeddings (OSE) of data sets into Euclidian spaces that preserve relative error~\cite{ANW14,Woo14,DBLP:journals/corr/LarsenN14,CW15}.  Strong positive results are known for high-dimensional linear kernels (via Johnson-Lindenstrauss~\citep{JL84,Woo14,DBLP:journals/corr/LarsenN14,LN17}), for polynomial kernels~\citep{ANW14}, and for any $M$-estimator with gradient between $1$ and $2$~\citep{CW15}, but has remained open for the Gaussian kernel.  Such strong guarantees are, for instance, required to prove results about regression on the resulting set since we may not know the units on different coordinates; additive error bounds do not make sense in directions which are linear combinations of several coordinates.

The obliviousness of the features (they can be defined without seeing the data, and in some cases are independent of the data size) are essential for many large-scale settings such as streaming or distributed computing where we are not able to observe all of the data at once.

Our results do not describe unrestricted OSEs, as are possible with polynomial kernels \citep{ANW14}. Rather our lower bounds show that any OSE must have the dimension depend on $n$ or $\mathcal{M}$.

\paragraph{Kernel $k$-means clustering.}
Kernel $k$-means~\citep{Gir02} aims to find a set of $k$ center points in $\Eu{H}_K$ minimizing the sum of squared kernel distances from the $\phi(x) \in \Phi(X)$ to the closest center point.

Typical approaches either use the full $O(n)$-sized representation of the center~\citep{Gir02} or heuristically approximate $\Eu{H}_K$ using the top $k$-eigenvectors of the Gram matrix $G$ (with no individual distance guarantees). In order to perform kernel $k$-means clustering in the former case, a recurring operation is to invoke the distance computation between the $k$ center points and $\phi(x)$. Due to the representation size of each center point, the operation takes at least time $\Omega(n)$.  If an approximate lifting map $\hp$ is used instead, the center points can be explicitly represented as a $m$-dimensional point, and the distance computation would take $O(dm)$ time with bounded relative error.  This also means the related sublinear algorithms such as \citep{ailon2009streaming} can be applied directly, with small space usage, which is not possible if one can only rely on the Gram matrix.

On the other hand, often these methods may use a representative data point $\phi(x) \in \Phi(X)$ instead of the mean of the included data points~\citep{DGK04}.  Then our upper bounds imply one can simply work in Euclidean $\R^m$ space, and have relative error guarantees on the overall cost function found. This still allows us to use spatial indexing or searching techniques such as LSH and k-d trees to speed up algorithms such as $k$-means++~\citep{AV07} or the \cite{Gon85} algorithm for kernel k-center clustering. 

\paragraph{Kernel distance matching.}
The kernel distance $D_K(X,Y)$ between two point sets provides a robust and powerful distance between objects $X$ and $Y$, for instance probability measures~\citep{smola,GBRSS12}, medical images of organs~\citep{DBLP:conf/miccai/DurrlemanPTA07,GlaunesJoshi:MFCA:06}, and general shapes~\citep{current}.  However this distance (a single scalar value) \emph{does not} imply or provide an alignment between the point sets (unlike other common integral probability measures, say like the Wasserstein family of distances e.g., earth-movers).  Embedding the point sets into $\R^m$, allows one to invoke powerful geometric approaches using Euclidian distance~\citep{SA12,AS14} to construct the \emph{matching} which approximately minimizes the pairwise kernel distance.

\section{Basic Bounds and Taylor Approximations}
\label{sec:basic}

For the remainder of the paper, it will be convenient to let $\Delta = (x-y)/\sigma$ be the scaled vector between some pair of points $x,y \in X$.   Define $D_K(\Delta) = D_K(x,y) = \sqrt{2 - 2 e^{\frac{1}{2} \|\Delta\|^2}}$, and also $K(\Delta) = K(x,y)$ and $\hat K(\Delta) = \hat K(x,y)$.

Using $t = O((1/\eps^2) \log(1/\delta))$ features for $\eps \in (0,1/2)$ and $\delta \in (0,1)$, then for any $\Delta \in \R^d$, the following PAC bound~\citep{SS15-UAI,rahimi2007random} holds
\begin{equation}\label{eq:additive}
\Pr\left[ \left| K(\Delta) - \hat K(\Delta) \right| \leq \eps\right] \geq 1-\delta.  
\end{equation}

Since $\hat K(x,x) = 1$, then $D_{\hat K}(\Delta)^2 = 2 - 2 \hat K(\Delta)$, and additive error bounds between $D_K(\Delta)^2$ and $D_{\hat K}(\Delta)^2$ follow directly.  
But we can also state some relative error bounds when $\|\Delta\|$ is large enough.  
\begin{lemma}
\label{lem:bigD-rel}
For each $\Delta \in \R^d$ such that $\|\Delta\| \geq 1$ and $m = O((1/\eps^2) \log(1/\delta))$ with $\eps \in (0, 1/10)$ and $\delta \in (0,1)$.  
Then with probability at least $1-\delta$, we have 
$
\frac{D_K(\Delta)}{D_{\hat K}(\Delta)} \in [1-\eps,1+\eps].
$
\end{lemma}
\begin{proof}
By choosing $m = O((1/\eps^2) \log (1/\delta))$ so that $|K(\Delta) - \hat K(\Delta)| \leq \eps/4$, via (\ref{eq:additive}), 
we have that $|D^2_K(\Delta) - D^2_{\hat K}(\Delta) | \leq \eps/2$.  
We also note that when $\|\Delta\| \geq 1$ then $K(\Delta) \leq \frac{1}{\sqrt{e}} \leq 0.61$.  
Hence $D^2_K(\Delta) \geq 2(1-0.61) = 0.78 \geq 0.5$, and we have that 
$
| D^2_K(\Delta) - D^2_{\hat K}(\Delta) | \leq \eps/2 \leq \eps D^2_K(\Delta).  
$
Then 
$| 1 - \frac{D^2_{\hat K}(\Delta)}{D^2_K(\Delta)} | \leq \eps$, 
implying
$1-\eps \leq \frac{D^2_{\hat K}(\Delta)}{D^2_K(\Delta)} \leq 1+\eps$.  
Taking the square root of all parts completes the proof via $\sqrt{1+\eps} < (1+\eps)$ and $\sqrt{1-\eps} > (1-\eps)$.  
\end{proof}

\paragraph{Basic bounds when $\|\Delta\| < 1$.}  When $\|\Delta\| \leq 1$, then 
a simple Taylor expansion, implies that
\[
\|\Delta\|^2 - \frac{1}{4} \|\Delta\|^4
\leq
D_K(\Delta)^2 = 2 - 2\exp(\|\Delta\|^2/2) \leq \|\Delta\|^2,
\]
and by $\frac{1}{4} \|\Delta\|^4 \leq \frac{1}{4} \|\Delta\|^2$ and a square root 
\begin{equation} \label{eq:near-linear}
0.86 \|\Delta\| \leq D_K(\Delta) \leq \|\Delta\|.
\end{equation}
Moreover, when $\|\Delta\| \leq 2 \sqrt{\eps}$ then 
\begin{equation} \label{eq:eps-linear}
(1-\eps) \|\Delta\|^2 \leq D_K(\Delta)^2 \leq \|\Delta\|^2.
\end{equation}

\paragraph{Useful expansions.}
We first observe that by $\cos(a)\cos(b) + \sin(a)\sin(b) = \cos(a-b)$ that
\[
\langle \hat f_i(x), \hat f_i(y) \rangle 
= 
\cos(\langle \omega_i, x \rangle) \cos(\langle \omega_i, y \rangle) + \sin(\langle \omega_i, x \rangle) \sin(\langle \omega_i, y \rangle) 
= 
\cos(\langle \omega_i, (x-y) \rangle).  
\]
Hence by $\langle \hat f_i(x), \hat f_i(x) \rangle = \cos(\langle \omega_i, 0 \rangle) = 1$ we have $D_{\hat K}(x,y)^2 = 2 - 2 \frac{1}{t} \sum_i^t \cos (\langle \omega_i, ( x - y ) \rangle)$.  

By the rotational stability of the Gaussian distribution we can replace $\langle \omega_i, ( x - y ) \rangle$ with $\omega_{i,x,y} \frac{\| x - y\|}{\sigma}$ where $\omega_{i,x,y} \sim \mathcal{N}(0,1)$.  It will be more convenient to write $\omega_{i,x,y}$ as $\omega_{i,\Delta}$, so $\langle \omega_i, ( x - y ) \rangle = \omega_{i, \Delta} \|\Delta\|$. 
Thus
$\langle \hat f_i(x), \hat f_i(y) \rangle = \cos(\omega_{i,\Delta} \|\Delta\|)$.  
Moreover, we can define
$
D_{\hat K}(\Delta) = D_{\hat K}(x,y) = \sqrt{2 - 2 \frac{1}{t} \sum_{i=1}^t \cos(\omega_{i,\Delta} \|\Delta\|)}
$.

Now considering 
\[
\frac{D_{\hat K}(\Delta)^2}{D_K(\Delta)^2} = \frac{1 - \frac{1}{t} \sum_{i=1}^t \cos(\omega_{i,\Delta} \|\Delta\|)}{1- e^{\frac{1}{2} \|\Delta\|^2}}, 
\]
the following Taylor expansion, for $\omega_{i,\Delta} \|\Delta\| \leq 1$, will be extremely useful:
\[
 \frac{\frac{1}{t} \sum_{i=1}^t \frac{1}{2} \omega^2_{i,\Delta} \|\Delta\|^2 }{\frac{1}{2}{\|\Delta\|^2} - \frac{1}{4} \|\Delta\|^4}
\geq
\frac{D_{\hat{K}}(\Delta)^2}{D_{K}(\Delta)^2}
\geq
\frac{\frac{1}{t} \sum_{i=1}^t \left( \frac{1}{2}\omega^2_{i,\Delta} \|\Delta\|^2 -  \frac{1}{24} (\omega^4_{i,\Delta} \|\Delta\|^4 )\right)}{\frac{1}{2}{\|\Delta\|^2}}.
\]
Simplifying gives
\begin{equation}     \label{equation:RelativeErrorObjective}
\frac{1}{1 - \frac{1}{2} \|\Delta\|^2} \left(\frac{1}{t} \sum_{i=1}^t\omega^2_{i,\Delta}\right)
    \geq \frac{D_{\hat{K}}(\Delta)^2}{D_{K}(\Delta)^2} \geq 
    \left(\frac{1}{t} \sum_{i=1}^t \omega^2_{i,\Delta}\right) - \frac{\| \Delta \|^2}{12 } \cdot \frac{1}{t} \sum_{i=1}^t \omega^4_{i,\Delta}.
\end{equation}

\paragraph{Roadmap.}
To understand the detailed relative error in $D_K(\Delta)$, what remains is the case when $\|\Delta\|$ is small.  As we will start to observe above, when $\|\Delta\|$ is small, then $D_K(\Delta)$ behaves like $\|\Delta\|$ and we can borrow insights from $\ell_2$ embeddings.  
Then combining the two cases (when $\|\Delta\|$ is large and when $\|\Delta\|$ is small) we can achieve ``for all bounds" either via simple union bounds, or through a special ``continuous'' form of net arguments when $X$ is in a bounded range.  Similarly, we will show near-matching lower bounds via appealing to near-$\ell_2$ properties or via net arguments.

\section{Lower Bounds and Relation to $\ell_2$ on Small Distances}
\label{sec:local-l2}

In this section we show that in the limit as the region containing $X$ shrinks, then all distances act like $\ell_2$.  This approach is enough for a lower bound, but does not contain the full case $\|\Delta\| \leq 1$, so is not enough for upper bounds.

\begin{lemma}  \label{lem:lim-err}
For scalar scaling parameter $\lambda$, 
$\displaystyle{
 \lim_{\lambda \rightarrow 0} \frac{D_{\hat K}(\lambda \Delta)^2}{D_K(\lambda \Delta)^2} = \frac{1}{t} \sum_{i=1}^t \omega^2_{i,\Delta}.
}$
\end{lemma}
\begin{proof}
Observe that $\omega_{i,\Delta} = \omega_{i,\lambda \Delta}$, for any $\lambda > 0$.  
Thus in equation (\ref{equation:RelativeErrorObjective}), $\lim_{\lambda \to 0} 1/(1-\frac{1}{2}\|\lambda \Delta\|^2)$ goes to $1$ so the left hand-side approaches $\frac{1}{t} \sum_{i=1}^t \omega_{i,\Delta}^2$.  
Similarly, $\lim_{\lambda \to 0} \|\lambda \Delta\|^2 /12$ goes to $0$, and the right-hand side also approaches $\frac{1}{t} \sum_{i=1}^t \omega_{i,\Delta}^2$.  
\end{proof}

If we fix $\Delta$ then $\omega_{i,\Delta}, 1 \leq i \leq t$ are i.i.d Gaussian variables with mean $0$ and standard deviation $1$. Thus $\sum_{i=1}^t\omega^2_{i,\Delta}$ is a $\chi^2$-variable with $t$ degrees of freedom.

This implies that when $\|x-y\|$ is small, $D_{\hat K}(x,y)$ behaves like a Johnson-Lindenstrauss (JL) random projection of $\|x-y\|$, and we can invoke known JL lower bounds.  

In particular, Lemma \ref{lem:lim-err} implies if the input data set $X \subset \R^d$ is in a sufficiently small neighborhood of zero, the relative error is preserved only when $\sum_{i=1}^t \omega^2_{i,x,y}\|\lambda(x-y)\|^2 \in [(1-\epsilon)\|\lambda(x-y)\|^2,(1+\epsilon)\|\lambda(x-y)\|^2]$ for all $x,y \in X$, and for all arbitrary $\lambda \in \R$. Which implies for arbitrary $x,y \in X$, and $\lambda \in \R$ that
\[
\sqrt{\sum^t_{i=1} \left|\omega_{i}\cdot \lambda(x-y)\right|^2 } 
= 
\sqrt{\sum^t_{i=1} \omega^2_{i,x,y}\lambda\|x-y\|^2 } 
\in 
\left[ \left(1-\eps\right)\lambda\left\| x -y \right\|,\left(1+\eps\right)\lambda\left\| x -y \right\| \right].  
\]

The far left hand side is in fact the norm $\| g(x)-  g(y)\|$ where $g(x)$ is the vector with coordinates $(\omega_1 \cdot \lambda x, ..., \omega_t \cdot \lambda x)$. Thus these are the exact conditions for relative error bounds on embedding $\ell_2$ via the Johnson-Lindenstrauss transforms, which gives the following.

\begin{lemma} \label{lem:RFF-JL}
 If for any $n,d > 0, X \subset \R^d$ s.t. $|X| = n$, using $t(n,\eps)$ pairs of random Fourier features, $\frac{D_{\hat{K}}(x,y)}{D_{K}(x,y)} \in [1 - \eps, 1 + \eps]$ with probability $1 - \delta$, then there exists a random linear embedding with $t(n,\eps)$ projected dimensions preserving the $\ell_2$-norm for all pairs $x,y \in S$ up to relative error with probability at least $1 - \delta$.
\end{lemma}

\begin{theorem}
There exists a set of $n$ points $X \subset \R^d$ so that $t = \Omega(\frac{1}{\eps^2} \log n)$ pairs of random features (hence $m=2t$ dimensions), for any $\eps \in (0,1/2)$, are necessary so for any $x,y \in X$ that $\frac{D_{\hat{K}}(x,y)}{D_{K}(x,y)} \in [1 - \eps, 1 + \eps]$.
\label{thm:eps-2logn-LB}
\end{theorem}
\begin{proof}
A lower bound of $\Omega(\frac{1}{\eps^2} \log n)$ projected dimensions for linear embeddings in $\ell_2$ is here~\citep{LN17}.
\end{proof}

\section{Relative Error Bounds For Small Distances and Small Data Sets}
\label{sec:small-data}

The Taylor expansion in equation (\ref{equation:RelativeErrorObjective}) and additive errors via equation (\ref{eq:additive}) 
are only sufficient to provide us bounds for $\|\Delta\| \leq O(\sqrt{\eps}/\log (1/\eps))$ or for $\|\Delta\| \geq 1$.  

\emph{The prior, published~\citep{chen2017relative}, version of this work had an error in this section (due to a sign error in the appendix) as pointed out by \cite{chengrelative23}.  Luckily, we can fix this with a more modern concentration bound.  }

We will use a Bernstein inequality~\citep{Vershynin_2026}[Cor 2.9.2] for $t$ iid random variables $Z_1, \ldots, Z_t$ which satisfy $\E[Z_i] = 0$ and have sub-exponential Orlicz norm $\|Z_i\|_{\psi_1} \leq K$ (that is, $\E[\exp(|Z_i|/K)] \leq 2$).  Then for any $\alpha \in (0,1]$ we have 
$
 \Pr\left( \left| \frac{1}{t} \sum_{i=1}^t Z_i \right| > \alpha \right) \leq 2 \exp(- t C \alpha^2 / K^2),
$
for an absolute constant $C$.
We can now provide this bound for the $\|x-y\| \leq \sigma$ case.

\begin{lemma}
\label{lem:smallD-rel}
If $\|x - y\| \leq \sigma$, $t = \Omega(\frac{1}{\eps^2} \log \frac{1}{\delta})$, then $\Pr\left(\frac{D_{\hat{K}}(x,y)}{D_K(x,y)} \in [1 - \eps, 1 + \eps]\right) \geq 1 - \delta$.
\end{lemma}
\begin{proof}
Recall that $\langle \hat f_i(x), \hat f_i(y) \rangle = \cos(\langle \omega_i, (x-y) \rangle)$ and $(1/2) D_{\hat K}(x,y)^2 = \frac{1}{t} \sum_{i=1}^t (1- \cos(\langle \omega_i, (x-y) \rangle))$.  
Then define random variable $X_i = 1- \cos(\langle \omega_i, (x-y) \rangle) = 1-\cos(\omega_{i,\Delta} \Delta)$, 
and $X = \frac{1}{t} \sum_{i=1}^t X_i$.  
Since $\E[X] = \E[D_{\hat K}(x,y)^2] = D_K(x,y)^2$, then $\E[X] = 1-\exp(-\frac{1}{2}\|\Delta\|^2)$. 

First we can upper-bound $X_i = 1-\cos (\omega_{i,\Delta} \Delta) \leq \omega_{i,\Delta}^2 \Delta^2/2$.  
Next we can lower-bound $\E[X] \geq \Delta^2/(2 \sqrt{e})$.  We obtain this using the mean value theorem on $v \in [0,1/2]$ so $1-\exp(-v) \geq v/\sqrt{e}$; then using $\Delta \in [0,1]$ we have $\E[X] = 1-\exp(-\Delta^2/2) \geq \Delta^2/(2 \sqrt{e})$.  
Hence we can create the mean-center random variable $Z_i = \frac{X_i - \E[X]}{\E[X]} = \frac{X_i}{\E[X]} - 1$ and bound $|Z_i| \leq \frac{X_i}{\E[X]} + 1 \leq \omega_{i,\Delta}^2 \sqrt{e} + 1$.  

Now we need to show how $Z_i$ satisfies the conditions to use the stated Bernstein inequality, in particular using $K = 4 \sqrt{e}$.  We have $\frac{|Z_i|}{K} \leq \frac{\sqrt{e} \omega_{i,\Delta}^2 + 1}{4 \sqrt{e}} = \omega_{i,\Delta}^2/4 + 1/4\sqrt{e}$.  Then we need the property $\E_x[\exp(\theta x^2)] = 1/\sqrt{1-2\theta}$ for $\theta < 1/2$.  
Putting this together we have 
\[
\E[\exp(|Z_i|/K)] 
 \leq e^{1/4\sqrt{e}} \E[\exp(\omega_i^2/4)]
 \leq e^{1/4\sqrt{e}} / \sqrt{1-1/2} 
 < 2,
\]
which implies $\|Z_i\|_{\psi_1} \leq K = 4 \sqrt{e}$ for any $\Delta \in [0,1]$.  

Now we can apply the Bernstein bound to obtain that
\[
\Pr\left(\frac{|D_{\hat{K}}(x,y) - D_K(x,y)|}{D_K(x,y)} \geq \eps \right) 
=
\Pr\left(\left| \frac{1}{t} \sum_{i=1}^t Z_i \right| \geq \eps \right) 
\leq
2 \exp( - t C \eps^2/K^2) 
\leq \delta; 
\]
where the last step holds with $t = C K^2 \frac{1}{\eps^2} \ln(1/\delta) = C 16e \frac{1}{\eps^2} \ln(2/\delta)$ to complete the proof.  
\end{proof}

Using Lemma \ref{lem:smallD-rel} (for $\|x-y\| \leq \sigma$) and with Lemma \ref{lem:bigD-rel} (for $\|x-y\| \geq \sigma$) together, we apply a union bound over all $n \choose 2$ pairs vectors from a set of $n$ vectors.

\begin{theorem}
\label{cor:eps-2logn-UB}
For any set $X \subset \R^d$ of size $n$, then $m = 2t = \Omega(\frac{1}{\eps^2} \log n)$ projected dimensions are sufficient so $\frac{D_{\hat K}(x,y)}{D_K(x,y)} \in [1-\eps,1+\eps]$ with high probability (e.g., at least $1-1/n$).  
\end{theorem}

\section{Relative Error Bounds for Low Dimensions and Diameter}
\label{sec:lowD}

Here we prove that the relative error bound holds for the infinitely many pairs of vectors of finite distance to each other, given that the number of dimensions is small. A common approach in subspace embeddings replaces $n$ with the size of a sufficiently fine net; given a smoothness condition, once the error is bounded on the net points, the guarantee is extended to the `gaps' in between.

On the other hand, the Gaussian kernel distance is non-linear, so it is not immediately clear how the above technique can apply. We begin with the Lipschitz constant of $D_{\hat{K}}(\cdot)^2$, with respect to the vector $\Delta$, \emph{not individual points in $\R^d$}. 
Then we develop a fine-grained structure and a net on the set of directions $\Delta/\|\Delta\|$ as long as $\|\Delta\|$ is small enough, using an object we call a $\lambda$-urchin.

\paragraph{Lipschitz bound.}
First we provided the needed Lipschitz bound with respect to $\Delta$.  
\begin{lemma}\label{lemma:Gradient}
 For any $\Delta \in \mathbb{R}^d$, $|\nabla D_{\hat{K}}(\Delta)^2| \leq  2\frac{1}{t} \sum^t_{i=1} \|\omega_i\|_1\|\omega_{i}\| \|\Delta \|$. 
\end{lemma}

\begin{proof}
 We denote by $\omega^{(j)}_i$ the $j$th coordinate of $\omega_{i}$; where recall $\omega_{i,\Delta} = \langle \omega_i, \Delta \rangle$.    
 \begin{align*}
  \left|\nabla D_{\hat{K}}(\Delta)^2\right|
  &= 
  2 \left| \frac{1}{t} \sum^t_{i=1}\sum^d_{j=1} \omega^{(j)}_{i} \sin (\langle \omega_{i}, \Delta\rangle)\right| 
  \leq  
  2  \frac{1}{t} \sum^t_{i=1}\sum^d_{j=1} |\omega^{(j)}_{i}| \left| \sin (\langle \omega_{i}, \Delta\rangle)\right|
  \\&\leq  
   2 \frac{1}{t} \sum^t_{i=1} \|\omega_{i,\Delta}\|_1 |\langle \omega_{i}, \Delta\rangle|
  \leq  
   2 \frac{1}{t} \sum^t_{i=1} \| \omega_{i} \|_1 \| \omega_{i} \| \| \Delta \|
\end{align*} 

\vspace{-.4in}
\end{proof}

\begin{corollary}
For any $c \geq 0$, over the region $\| \Delta \| \leq c$, the Lipschitz constant of $D_{\hat{K}}(\Delta)^2$ is bounded above by $O(c \cdot \sqrt{d} \log (d/\delta))$ with probability at least $1 - O(\delta)$.
\label{corollary:lipschitz}
\end{corollary}
\begin{proof}
We can bound any coordinate $\omega_i^{(j)}$ of $\omega_i$ so that $|\omega_i^{(j)}| \leq O(\log \frac{1}{\delta})$ with probability at least $1-\delta$.  By a union, bound the all coordinates $|\omega_i^{(j)}| \leq O(\log \frac{d}{\delta})$ with probability at least $1-\delta$.  
So the gradient is bounded by 
$2 \|\Delta\| \frac{1}{t} \sum_{i=1}^t \|\omega_i\|_1 \|\omega_i\|
\leq
\|\Delta\| \sqrt{d} O(\log \frac{1}{\delta})
\leq 
O(c \cdot \sqrt{d} \log \frac{1}{\delta})$, 
which also bounds the Lipschitz constant.
\end{proof}

In case $c = \frac{\sqrt{\eps}}{2 \ln(4/\eps \delta)}$, the Lipschitz constant is $O(\sqrt{\eps d} \log \frac{d}{\delta})$ with probability at least $1-\delta$.

\paragraph{Fine-grained small distance structure.}
We now analyze equation (\ref{equation:RelativeErrorObjective}).  We first state a standard bound on $\chi^2$ random variables $\sum_{i=1}^t \omega_{i,\Delta}^2$, and then show how to bound the other terms.  

\begin{lemma}   \label{lem:OmegaChi}
For $\eps\in (0,1)$, $\delta \in (0,\frac{1}{2})$, if $t \geq 8 \frac{1}{\eps^2} \ln (2/\delta)$ then 
$
\Pr \left[\frac{1}{t}\sum_{i=1}^t\omega^2_{i,\Delta} \notin [1 - \eps, 1 + \eps] \right] \leq \delta.  
$
\end{lemma}
\begin{proof}
Here we use Lemma 1 from \cite{10.2307/2674095}; if $X$ is a $\chi^2$ random variable with $t$ degrees of freedom
$ 
\Pr[t - 2 \sqrt{tx} \leq X \leq t + 2 \sqrt{tx} + 2x] \geq 1 - 2e^{-x}.  
$
We can set $x = \frac{1}{8} t \eps^2$ then 
$t - 2 \sqrt{tx} = t - \eps t/\sqrt{2}$, and 
$t + 2 \sqrt{tx} + 2x = t + \eps t/\sqrt{2} + \frac{1}{4} \eps^2 t < t + \eps t$.
Also, 
$2e^{-x} = 2 e^{-\frac{1}{8} t \eps^2} = 2 e^{- \ln (2/\delta)} = \delta/2 \leq \delta$ for $\delta \leq 1/2$.  
Therefore, $\frac{1}{t}\sum_{i=1}^t\omega^2_{i,\Delta} \notin [1 - \eps, 1 + \eps]$ with probability at most $\delta$.    
\end{proof}

Now to bound the other parts ($\|\Delta\|^2/2$ and the term containing $\omega^4_{i,\Delta}$) of equation (\ref{equation:RelativeErrorObjective}) requires a further restriction on $\|\Delta\|$.  

\begin{lemma}   \label{lem:SpikePreservation}
For $\eps\in (0,1)$ and $\delta \in (0,2/5)$, if $\|\Delta\| \leq \frac{\sqrt{\eps}}{2\ln (4/\eps \delta)}$ for a constant $C$, and $t \geq 18 \frac{1}{\eps^2} \ln (4/\delta)$, then with probability at least $1 - \delta$, 
 for all $\lambda \in [0,1]$ we have 
$
\frac{D_{\hat{K}}(\lambda \cdot \Delta)^2}{D_{K}(\lambda \cdot \Delta)^2} \in [1 - \eps, 1 + \eps].
$
\end{lemma}

\begin{proof}
If $\omega$ is a standard Gaussian variable, then $|\omega| \leq  \sqrt{2 \ln (1/\delta')}$ with probability at least $1 - \delta'$. 
Using $\delta' = \delta/2t$, then applying a union bound ensures that  (using $\ln(4/\delta) < 1/\delta$ for $\delta < 2/5$)
\[
\omega_{i,\Delta} 
\leq 
\sqrt{2 \ln (2t/\delta)} 
= 
\sqrt{2 \ln(\frac{16}{\delta \eps^2} \ln(4/\delta))}
\leq
\sqrt{2 \ln(\frac{16}{\delta^2 \eps^2})}
=
2\sqrt{\ln(4/\delta\eps)},
\]
 for $t$ such random variables with probability at least $1-(\delta' t) = 1-\delta/2$.  
This means, if 
$\|\Delta\| \leq \frac{\sqrt{\eps}}{2 \ln (4/\eps \delta)}$
then
$\omega_{i,\Delta} \|\Delta\| \leq \sqrt{\frac{\eps}{\ln (4/\eps \delta)}}$ 
with probability at least $1 - \delta/2$ for $t$ such random variables.  Also then each $\omega_{i,\Delta} \|\Delta\| \leq 1$ satisfies the conditions for ($\ref{equation:RelativeErrorObjective}$).

Then using $\|\Delta\| \leq \frac{\sqrt{\eps}}{\log(1/\eps\delta)}$ and $\omega_{i,\Delta} \leq 2\sqrt{\ln(4/\eps\delta)}$ we can bound the last term in (\ref{equation:RelativeErrorObjective}) as
\[
\frac{\| \Delta \|^2}{12} \cdot \frac{1}{t} \sum_{i=1}^t \omega^4_{i,\Delta} 
= 
\frac{\| \Delta \|^2}{12} \cdot \left(2 \sqrt{\ln(4/\eps \delta)}\right)^4 
\leq 
\frac{\eps}{12 \cdot 4 \ln^2 (4/\eps \delta)} \cdot 16 \ln^2(4/\eps \delta)
=
\frac{\eps}{3}. 
\]
Then along with Lemma~\ref{lem:OmegaChi}  (error $1-\frac{2\eps}{3}$) and RHS of (\ref{equation:RelativeErrorObjective}) (error $\frac{\eps}{3}$), we have $\frac{D_{\hat{K}}(\Delta)^2}{D_K(\Delta)^2} \geq 1 - \eps$. 
 
Similarly, Lemma~\ref{lem:OmegaChi} and $\frac{1}{2}\|\Delta\|^2 < \frac{\eps}{8 \ln^2 (4/\eps \delta)} < \frac{\eps}{8}$ imply the LHS of (\ref{equation:RelativeErrorObjective}) is bounded above by $(1+\frac{2}{3}\eps)(1/(1-\frac{\eps}{8})) \leq 1 + \eps$ with probability at least $1-\delta/2$.  Thus, we have $\frac{D_{\hat{K}}(\Delta)^2}{D_{K}(\Delta)^2} \in [1 - \eps, 1 + \eps]$.
 
For $\frac{D_{\hat{K}}(\lambda \Delta)^2}{D_{K}(\lambda \Delta)^2} \in [1 - \eps, 1 + \eps]$, note that the above analysis still holds if we scale $\|\Delta\|$ to be smaller, i.e. as long as $\lambda \in [0,1]$.  In particular, $\omega_{i, \Delta}$ is unchanged by the scaling $\lambda$.  
\end{proof}

\paragraph{Scaled net argument.}
We can now provide a net argument for a relative error bound for all small $\Delta$.  Intuitively, what separates typical net arguments from ours is the scaling $\lambda$ in Lemma~\ref{lem:SpikePreservation}; our `net' contains a set of line segments extending from the origin, which we call a \emph{$\lambda$-urchin}.  

\begin{lemma}
If $t = \Omega(\frac{d}{\eps^2} \log \left( \frac{d}{\eps}\frac{1}{\delta} \right))$, then with probability at least $1 - \delta$, for all $\Delta$ such that $\|\Delta\| \leq \frac{\sqrt{\eps}}{2\ln(4/\eps \delta)}$, then $\frac{D_{\hat{K}}(\Delta)^2}{D_K(\Delta)^2} \in [1 - \eps, 1 + \eps]$.
\end{lemma}
\begin{proof}
The proof will first consider distances $\Delta$ such that $\{ \Delta : \|\Delta\| = R_\eps \}$ where $R_\eps = \frac{\sqrt{\eps}}{2\ln(4/\eps \delta)}$, 
and then generalize to smaller distances using Lemma \ref{lem:SpikePreservation} and a construction we call a $\lambda$-urchin.  

\noindent\textbf{Fixed distance case: }  
Consider two points $\Delta_1, \Delta_2$ from the surface $\{ \Delta : \|\Delta\| = R_\eps\}$.
If $\|\Delta_1 - \Delta_2\| \leq \frac{\sqrt{\eps}}{\sqrt{d} \log \frac{1}{\delta}} R_\eps^2$ then Corollary~\ref{corollary:lipschitz} implies
\begin{align*}
 \left| D_{\hat{K}}(\Delta_1)^2 - D_{\hat{K}}(\Delta_2)^2 \right|
& \leq  
O(\sqrt{\eps d} \log\frac{1}{\delta}) \cdot \|\Delta_1 - \Delta_2\|
\leq 
O(\sqrt{\eps d}\log\frac{1}{\delta}) \cdot \frac{\sqrt{\eps}}{\sqrt{d} \log\frac{1}{\delta}} R_\eps^2
\\ & =
O(\eps \cdot R_\eps^2)
=
O(\eps) \cdot D_K(\Delta_1)^2.  
\end{align*}

Now let $\Gamma_\gamma$ be a $\gamma$-net  over 
$\{ \Delta : \|\Delta\| = R_\eps \}$ 
where 
$\gamma \leq \frac{\sqrt{\eps}}{\sqrt{d} \log \frac{d}{\delta}} R_\eps^2$. 
For any 
$\Delta_1 \in \{\Delta : \|\Delta\| = R_\eps \}$, 
there exists $\Delta_2 \in \Gamma_\gamma$ such that 
$\|\Delta_1 - \Delta_2\| \leq \gamma$. 
Then the above implies
\begin{equation}
(1 - O(\eps)) D_{\hat{K}}(\Delta_2)^2 
\leq 
D_{\hat{K}}(\Delta_1)^2 
\leq 
(1 + O(\eps))D_{\hat{K}}(\Delta_2)^2.  
\label{eq:netgaperror}
\end{equation}

By the triangle inequality, equation (\ref{eq:near-linear}), and  $\sqrt{d} \log \frac{d}{\delta} \cdot 2 \ln(4/\eps \delta) > 1$, we have  
\begin{eqnarray}  \label{eq:trineterror}
|D_K(\Delta_1) - D_K(\Delta_2)|
&\leq &
D_K(\Delta_1,\Delta_2) 
\leq 
\|\Delta_1 - \Delta_2\| 
\leq  
\gamma   
\\ & \leq &
\frac{\sqrt{\eps}}{\sqrt{d} \log \frac{d}{\delta}} R_\eps^2
=
\eps
\frac{1}{\sqrt{d} \log \frac{d}{\delta} \cdot 2 \ln(4/\eps \delta)} R_\eps
< 
\eps \cdot O(D_K(\Delta_1)).   \nonumber
\end{eqnarray}

We will choose $t = \Omega(\frac{1}{\eps^2} \log \frac{|\Gamma_\gamma|}{\delta})$ so the following holds over $\Gamma_\gamma$ with probability at least $1-\delta$
\begin{equation}
(1 - O(\eps)) D_{K}(\Delta_2)^2 \leq D_{\hat{K}}(\Delta_2)^2 \leq (1 + O(\eps))D_{K}(\Delta_2)^2.  
\label{eq:neterror}
\end{equation}

These equations (\ref{eq:netgaperror}),  (\ref{eq:neterror}), and (\ref{eq:trineterror}) show, respectively that the ratios
$\frac{D_{\hat K}(\Delta_1)}{D_{\hat K}(\Delta_2)}$,
$\frac{D_{\hat K}(\Delta_2)}{D_{K}(\Delta_2)}$, and
$\frac{D_{K}(\Delta_2)}{D_{K}(\Delta_1)}$
are all in $[1+O(\eps), 1-O(\eps)]$; hence we can conclude
\begin{equation}
\label{eq:error-on-sphere}
|D_K(\Delta_1) - D_{\hat{K}}(\Delta_1)| \leq O(\eps) \cdot D_{K}(\Delta_1).  
\end{equation}
Which are in turn $1 \pm O(\eps)$ relative error bounds for the kernel distance, over $\{ \Delta : \|\Delta\| = R_\eps \}$.

\noindent\textbf{All distances case:  }
For the region $\{ \Delta : \|\Delta\| < R_\eps \}$, consider again $\Gamma_\gamma$. For each net point $p \in \Gamma_\gamma$ we draw a line segment from $p$ to the origin, producing the set of line segments $\bar{\Gamma}_\gamma$, that we call the \emph{$\gamma$-urchin}.
By Lemma~\ref{lem:SpikePreservation}, and $t = \Omega(\frac{1}{\eps^2} \log \frac{|\Gamma_\gamma|}{\delta})$, with probability at least $1-\delta$, we have relative error bounds for the Gaussian kernel distance over the $\gamma$-urchin.

Now for any $\lambda \in (0,1)$, consider the intersection $\{ \Delta : \|\Delta\| = \lambda R_\eps \} \cap \bar{\Gamma}_\gamma$. We see that the $\gamma$-urchin induces a net over $\{ \Delta : \|\Delta\|= \lambda R_\eps\}$. Due to scaling we can see that, in fact, it is a $(\lambda \gamma)$-net.   
So the distance between any point in $\{ \Delta : \|\Delta\| = \lambda R_\eps\}$ and the closest net point is bounded above by  
$\frac{ \lambda \sqrt{\eps}}{\sqrt{d} \log \frac{d}{\delta}} R_\eps^2$. 
From Corollary \ref{corollary:lipschitz}, the Lipschitz constant is now $O(\lambda \cdot  \sqrt{\eps d})$.

By arguments similar to those leading to (\ref{eq:error-on-sphere}) we obtain, for any $\Delta_1 \in \{ \Delta : \|\Delta\| = \lambda R_\eps\}$
\begin{equation}
|D_K(\Delta_1) - D_{\hat{K}}(\Delta_1)| \leq O(\eps) \cdot \lambda \cdot R_\eps \leq O(\eps) \cdot D_{K}(\Delta_1).    
\label{eq:errorlambdasphere}
\end{equation}
Since this holds for all $\lambda \in [0,1]$, we obtain relative error bounds over $\{ \Delta : \|\Delta\| \leq R_\eps\}$.

The size of $\Gamma_\gamma$ is bounded above by 
$O((R_\eps \frac{1}{\gamma})^d) = 
O((R_\eps \cdot \frac{\sqrt{d} \log \frac{d}{\delta}}{\sqrt{\eps}}  \frac{1}{R_\eps^2})^d)
=
O((\frac{\sqrt{d} \log \frac{d}{\delta}}{\sqrt{\eps}}  \frac{1}{R_\eps})^d)
=
O((\frac{\sqrt{d} \log \frac{d}{\delta} \log \frac{1}{\eps \delta}}{\eps})^d)$.
It is sufficient to have 
$t = O(\frac{1}{\eps^2} \log \frac{|\Gamma_\gamma|}{\delta}) 
= 
O(\frac{d}{\eps^2} \log \frac{d}{\eps \delta})$ 
so that relative error holds over the $\gamma$-net and the $\gamma$-urchin simultaneously, which imply (\ref{eq:errorlambdasphere}) and (\ref{eq:error-on-sphere}), with probability at least $1 - \delta$.
\end{proof}

\begin{corollary} \label{cor:rel-small-dist}
 If $t = \Omega(\frac{d}{\eps^2} \log \frac{d}{\eps \delta} )$, then for all $\Delta$ such that $\|\Delta\| \leq 1$, $\frac{D_{\hat{K}}(\Delta)}{D_K(\Delta)} \in [1 - \eps, 1 + \eps]$ with probability at least $1 - \delta$.
\end{corollary}
\begin{proof}
Consider the region $1 \geq \|\Delta\| > \frac{\sqrt{\eps}}{2\ln (4/\eps \delta)}$. The Lipschitz constant is bounded above by $O(t \sqrt{d} \log \frac{d}{\delta})$ by Corollary \ref{corollary:lipschitz}, so we only need a $\gamma$-net where $\gamma \leq \frac{\eps^2}{t\sqrt{d}\log \frac{d}{\delta}}$ to give relative error by standard net arguments. The size of this net is at most $\left( \frac{\sqrt{d}\log (d/\delta)}{ \eps^2} \right)^d$, so again it suffices to set $t = O(\frac{d}{\eps^2} \log \frac{d}{\eps \delta} )$ for our embeddings as above.
\end{proof}
Combined with Lemma \ref{lem:bigD-rel} for $\|\Delta\| > 1$ we obtain:

\begin{theorem}
 If $t = \Omega\left(\frac{d}{\eps^2} \log \left( \frac{d}{\eps} \frac{\mathcal{M}}{\delta}\right)\right)$, then for any $\mathcal{M} \geq 0$, $\frac{D_{\hat{K}}(x,y)^2}{D_K(x,y)^2} \in [1 - \eps, 1 + \eps]$ holds for all $x,y \in \R^d$ such that $\|x - y\|/\sigma\leq \mathcal{M}$ with probability at least $1 - \delta$.
\end{theorem}
\begin{proof}
Set 
$t = \Omega(\frac{d}{\eps^2} \log \frac{d}{\eps \delta}) 
+ 
\Omega(\frac{d}{\eps^2}d \log \frac{d \mathcal{M}}{\eps \delta}) 
= 
\Omega\left(\frac{d}{\eps^2} \log \left( \frac{d}{\eps} \frac{\mathcal{M}}{\delta}\right)\right)$ 
to account for both cases $\|\Delta\| = \frac{\|x-y\|}{\sigma} \leq 1$ and $1 \leq \frac{\|x-y\|}{\sigma} \leq \mathcal{M}$, respectively.
\end{proof}

\section{Lower Bounds for Low Dimensions}
\label{sec:LowDimLB}

When is $n$ is unbounded, a recent paper~\citep{SS15} implies that, even for small $d$, $D_{\hat K}$ cannot $(1 + \eps)$-approximate $D_K$ unless $\mathcal{M}$ is bounded.
Here we provide an explicit and \emph{general} lower bound depending on $\mathcal{M}$ and $d$ that matches the our upper bound up to a $O( \log \frac{1}{\eps})$ factor.  

First we need the following general result \citep{Alon:2003:PRE:2651487.2651714}[Theorem 9.3] related to embedding to $\ell_2$.
Let $B$ be an $n \times n$ real matrix with $b_{i,i}=1$ for all $i$ and $|b_{i,j}| \leq \eps$ for all $i \neq j$. If the rank of $B$ is $r$, and $\frac{1}{\sqrt{n}} < \eps < 1/2$, then
$
r \geq \Omega(\frac{1}{\eps^2 \log (1/\eps)}  \log n).  
$
Geometrically, $r$ is the minimum number of dimensions that can contain a set of $n$ near-orthogonal vectors. Indeed, any set $S$ of $n$ near-orthogonal vectors can be rotated to form the rows of a matrix of the form of $B$, and the rank is then the lowest number of dimensions that contain $S$.

\begin{lemma}	\label{lem:subspaceKernelLB}
Given $\mathcal{M} \geq 0$, let $B_\mathcal{M}(0)$ be the ball in $\mathbb{R}^d$ centered at the origin with radius $\mathcal{M}$. Let $h : \mathbb{R}^d \rightarrow \mathbb{R}^t$ be a mapping such that for any $x \neq y \in B_\mathcal{M}(0)$  we have $|K(x,y) - h(x)\cdot h(y)| \leq \eps \leq \frac{1}{4}$.
Then with sufficiently large $\mathcal{M}$, $t = \Omega(\frac{d}{\eps^2 \log  (1/\eps)} \log (\frac{\mathcal{M}}{ \log (1/\eps)}))$.
\end{lemma}
\begin{proof}
Consider a subset $S \subset \mathbb{R}^d$ in $B_{\mathcal{M}}(0)$ so for all $x,y \in S$, with $x\neq y$, we have $\|x-y\| \geq  \sigma \sqrt{2 \log \frac{1}{\eps}}$. 
Then for any $x,y \in S$, $K(x,y)=\exp(-\frac{\|x - y\|^2}{2\sigma^2}) \leq \eps$.  In particular, define $S$ as the intersection of $B_{\mathcal{M}}(0)$ with an orthogonal grid of side length $\sigma \sqrt{2 \log(1/\eps)}$; it has size $\Omega\left( \left(\frac{\mathcal{M}}{\log (1/\eps)} \right)^d\right)$.
	
For any $x,y \in S$, $|h(x) \cdot h(y)| \leq 2\eps$, and also $|\{h(s) \mid s \in S\}| = |S|$. Then \cite{Alon:2003:PRE:2651487.2651714}[Theorem 9.3] implies the dimension of $h$ must be $t = \Omega(\frac{1}{\eps^2 \log (1/\eps)}  \log |S|) = \Omega(\frac{d}{\eps^2 \log  (1/\eps)} \log (\frac{\mathcal{M}}{\log (1/\eps)}))$.
\end{proof}

\begin{theorem}  	\label{thm:subspaceLB}
Given $\mathcal{M} \geq 0$, let $B_\mathcal{M}(0)$ be the ball in $\R^d$ centered at the origin with radius $\mathcal{M}$. Let $h : \R^d \rightarrow \mathbb{R}^t$ be a mapping such that for any $x,y \in B_\mathcal{M}(0)$  we have $1-\eps \leq \frac{D_K(x,y)}{\|h(x) - h(y)\|}  \leq 1+\eps$ with $\eps \leq \frac{1}{4}$.  Restrict that for any $x \in \R^d$ that $\|h(x)\| = 1$.  If $\mathcal{M}$ is sufficiently large, $t = \Omega(\frac{d}{\eps^2 \log  (1/\eps)} \log (\frac{\mathcal{M}}{ \log (1/\eps)}))$.
\end{theorem}
\begin{proof}
Consider a set (as in proof of Lemma \ref{lem:subspaceKernelLB}) $S \subset B_{\mathcal{M}}(0)$.  
If for all $x,y \in S$ we have $1-\eps \leq \frac{D_K(x,y)}{\|h(x) - h(y)\|}  \leq 1+\eps$, then it implies 
\[
|D_K(x,y)^2 - \|h(x) - h(y)\|^2| \leq \Theta(\eps) D_K(x,y)^2 \leq \Theta(\eps),
\] 
since $D_K(x,y) < 2$.  
Expanding $D_K(x,y)^2 = 2 - 2 K(x,y)$ and $\|h(x) - h(y)\|^2 = 2 - 2 \langle h(x), h(y) \rangle$ implies that 
$|K(x,y) - \langle h(x), h(y) \rangle| \leq \Theta(\eps)$ as well.  However, Lemma \ref{lem:subspaceKernelLB} implies that for sufficiently small $\eps$ (adjusting the constant in $\Theta(\eps)$) that we require the $t = \Omega(\frac{d}{\eps^2 \log  (1/\eps)} \log (\frac{\mathcal{M}}{ \log (1/\eps)}))$.  
\end{proof}

This implies the impossibility of fully embedding into $\ell_2$ the Gaussian kernel distance over the \emph{entire} $\R^d$, i.e. for an infinite number of points, answering a question raised by \cite{SS15}. 
This argument can also extend to show a dependency on $d \log\mathcal{M}$ is inevitable when we do not have a bound on $n$.


\section{Discussion}  
\label{sec:discuss}
We have demonstrated theoretically tight relative error for kernel distance using random Fourier features, indicating tighter approximations for several important learning applications. In the following, we make some further remarks on the implications of our results, and then also empirically observe these properties.  

\subsection{Implications in Learning and Analysis}
In addition to the applications of our bounds to 
oblivious kernel embeddings, 
kernel $k$-means clustering, and
kernel distance matching
that we discussed in Section \ref{sec:impl}, we mention a couple more below.  

\paragraph{Geometric approximation in learning and mining.}
Our results show that random feature mappings allow for a finer notion of approximating the geometry of RKHS than previously known.
In particular, our low-dimensional bounds in Section \ref{sec:lowD}, imply that if an object $U \subset \R^d$ (such as a non-linear decision boundary) and training data $S \subset \R^d$ both lie within a ball with finite radius $\mathcal{M}$, then for any point $x \in S$, the minimum kernel distance between $U$ and $x$ is approximately preserved in the random feature space as $\min_{y \in U} \|\phi(x) - \phi(y)\|$.   For instance ``large-margin'' techniques and analyses~\citep{TJHA05} condition on the margin $\gamma = \max_{x \in S} \min_{y \in U} \|\phi(x) - \phi(y)\|$ being large, so we also preserve relative errors on this margin.  
This suggests better performance guarantees for kernelized learning large-margin techniques, and those involving the minimization of $\ell_2$ distances, such as in kernel SVM (hinge-loss) and in kernel PCA (recovery error); see Section \ref{sec:exp}.  

\paragraph{Gram matrix approximation.}
The approximation error of inner products is proportional to the approximation error of distances. This is because both $\phi$ and $\hp$ map every input point to a unit vector; thus $D_K(x,y)^2 = 2 - 2K(x,y)$ and $D_{\hat K}(x,y)^2 = 2 - 2 \hat{K}(x,y)$, for any distinct $x,y \in \R^d$. Therefore $|K(x,y) - \hat{K}(x,y)|$ is the same as $\frac{1}{2} | D_{K}(x,y)^2 - D_{\hat{K}}(x,y)^2 |$.  Hence approximation error of the Gram matrix is bounded in terms of the sum of pairwise squared distances
\[ 
\| G - \hat{G} \|_1 
\leq 
\frac{1}{2} \sum_{x \in X} \sum_{y \in X} |D_{K}(x,y)^2 - D_{\hat{K}}(x,y)^2|  
\leq
\eps \sum_{x \in X} \sum_{y \in X} D_K(x,y)^2,
\]
with high probability, when $m$ is set for the appropriate data setting in our bounds.  
Thus we have in some sense sharper bounds on approximating the Gram matrix.
%

\subsection{Remark on Lower Bound in $n$}
A new result of \cite{LN17} provides a $t = \Omega(\frac{1}{\eps^2} \log n)$ lower bound for even \emph{non-linear} embeddings of a size $n$ point set in $\R^d$  into $\R^t$ that preserve distances within $(1 \pm \eps)$.  It holds for any $\eps \in (1/\min\{n,d\}^{0.4999},1)$.  Since, there exists an isometric embedding of any set of $n$ points in any RKHS into $\R^n$, then this $t = \Omega(\frac{1}{\eps^2} \log n)$ lower bound suggests that it applies to $\hat \phi$ and $\tilde \phi$ or any other technique, for almost any $\eps$.  However, it is not clear that \emph{any} point set (including the ones used in the strong lower bound proof~\citep{LN17}), can result from an isomorphic (or approximate) embedding of RKHS into $\R^n$.  Hence, this new result does not immediately imply the lower bound we show in Section \ref{sec:local-l2}.  

Moreover, the proof of Theorem \ref{thm:eps-2logn-LB} retains two points of potential interest.  First it holds for a (very slightly) larger range of $\eps \in (0,1)$.  Second, Lemma \ref{lem:RFF-JL} highlights that at very small ranges, $\hat \phi$ is indistinguishable from the standard JL embedding.


\section{Empirical Demonstration of Relative Error}
\label{sec:exp}

We demonstrate that relative error actually results from the $\hat \phi$ kernel embeddings in two ways.  First we demonstrate relative error bounds for kernel PCA.  Second we show this explicitly for pairwise distances in the embedding.

\subsection{Relative Error for Kernel PCA}

When PCA is applied to approximate a data matrix in practice, the allowed approximation error is often chosen to be a small but constant (e.g. $10\%$) fraction of the total variance. Our results imply relative error in the approximation of the total variance, so we can also show relative error in typical cases of performing kernel PCA with Gaussian kernels using Random Fourier Features.

We consider two ways of running kernel PCA on the USPS data.  By default we use the first $n=2000$ data points in $\R^d$ for $d=256$, the first $n/10$ data points of each digit.  
In the first way, we create the $n \times n$ (centered) gram matrix $G$ of all inner products, and then use the top $k$ eigenvectors to describe the best subspace of RKHS to represent the data; this is treated as a baseline.  Second we embed each point into $\R^m$ using $\hp$, generating an $n \times m$ matrix $Q$ (after centering).  The top $k$ right singular values $V_k$ of $Q$ describe the kernel PCA subspace.

Error in PCA is typically measured as the sum of squared residuals, that is for each point $q \in Q \subset \R^m$, its projection onto $V_k$ is $V_k^T V_k q$, and its residual is $r_q = \|q - V_k^T V_k q\|^2$.  Thus $r_q$ is precisely the squared kernel distance between $q$ and its projection.  And then the full error is $\hat R_k = \|Q - V_k^T V_k Q\|_F^2 = \sum_{q \in Q} \|q - V_k^T V_k q\|^2$.  
For the non-approximate case, it can be calculated as the sum of eigenvalues in the tail $R_k = \sum_{i=k+1}^n \lambda_{i}$.  

Given $R_k$ and $\hat R_k$ we can measure the relative error as $\hat R_k / R_k$.  Our analysis indicates this should be in $[1-\eps, 1+\eps]$ using roughly $t = C/\eps^2$ features, where $C$ depends on $n$ or $d \log \mathcal{M}$.  To isolate $\eps$ we calculate $|\frac{\hat R_k}{R_k} - 1|$ averaged over $10$ trials in the randomness in $\hp$.  
This is shown in Figure \ref{tab:1} using $k = 40$, with $\sigma \in \{4,8,16\}$ and varying $t \in \{50, 100, 200, 400, 800\}$.  
We observe that our measured error decreases quadratically in $t$ as expected.  Moreover, this rate is stable as a function of $\sigma$ as would be expected where the correct way to quantify error is the relative error we measure.

\begin{figure}[h]

\begin{minipage}{0.51\textwidth}
  \centering
\includegraphics[width=\textwidth]{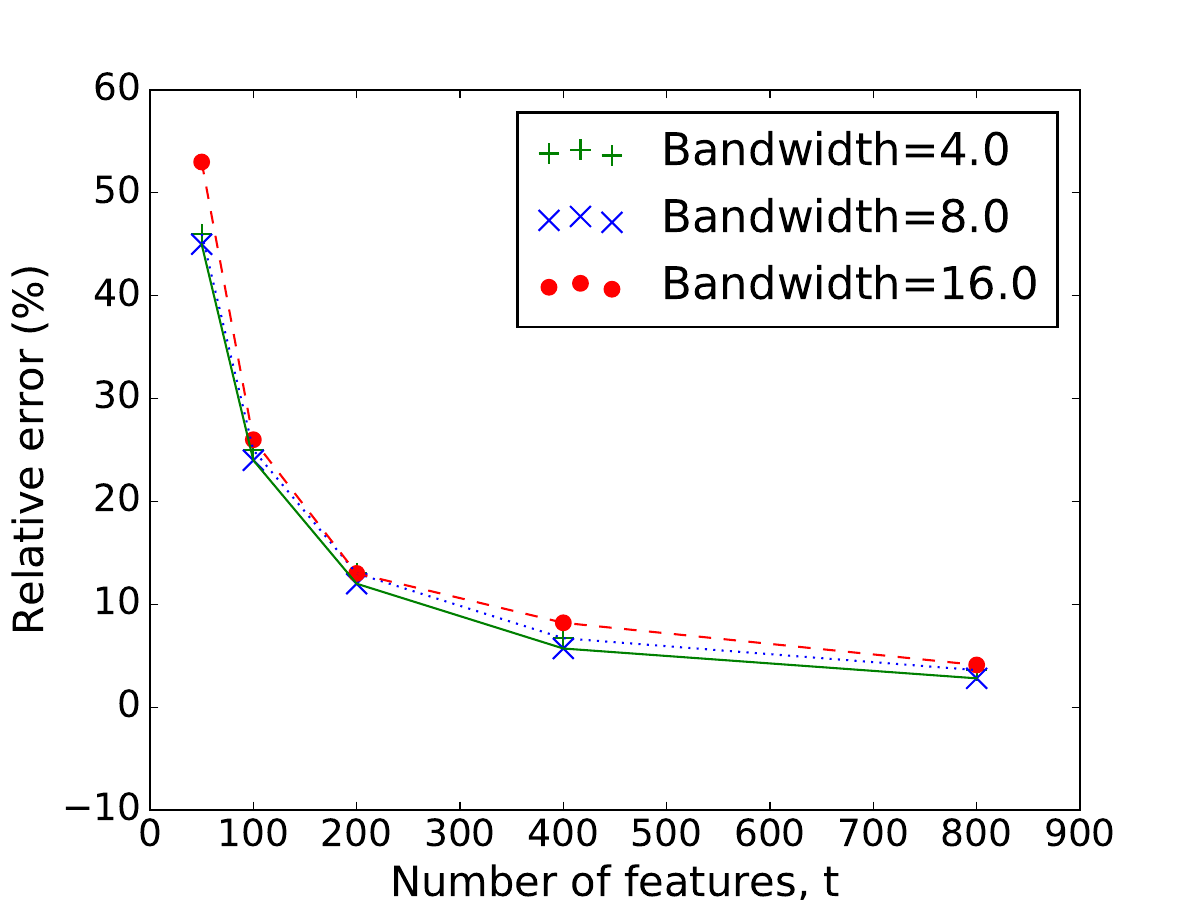}  
\end{minipage}
\begin{minipage}{0.43\textwidth}
  \centering
\begin{tabular}{cr||r|r|r}
		\hline
		&    & $\sigma=4$ &$\sigma=8$ & $\sigma=16$\\
		\hline
\multicolumn{2}{c||}{Baseline} & 1667.1 & 882.5 & 206.1\\
		\hline
		& 50  & 897.6  & 489.2 & 96.9 \\
		& 100 & 1257.6 & 666.5 & 152.2 \\
	$t$ & 200 & 1453.6 & 776.4 & 178.7 \\
		& 400 & 1554.7 & 831.5 & 189.2 \\
		& 800 & 1606.8 & 857.3 & 197.6 \\
		\hline
	\end{tabular}
\end{minipage}

  \caption{\sffamily Relative error $|\frac{\hat{R}k}{R_k}-1|$ in \% , against $t$, with $n=2000$, $k=40$ and different bandwidths. Relative error is roughly stable across different values of $\sigma$, and consistently reduced by increasing $t$. \label{tab:1}}
\end{figure}

\subsection{Pairwise Demonstrations of Relative Error}

Here we provide simulations that confirm our theoretical findings.
We randomly generate pairs of points $(x_1,y_1)$ \ldots $(x_n,y_n)$ with varying $\ell_2$ distance $\|x_i - y_i\|$; in particular, $x_i$ is a random point in a ball or radius $500$ and $y_i$ is generated to be a random point in the sphere $\|x-y\| = r_i$ where $r_1, ..., r_n$ follow a geometric distribution, ranging from approximately $10^{-4}$ to $10^4$.

In Figure~\ref{fig:VaryT+Dist}(left), for different values of $t$ (the number of features) we generate a fresh sequence of $2000$ random pairs, and record the maximum relative error $\eps_{\mathrm{max}} = \max_i \frac{D_K(x_i,y_i)}{\|\phi(x_i) - \phi(y_i)\|}$. The graph shows that $t$ is roughly proportional to $\eps^{-2}_{\mathrm{max}}$.

In Figure~\ref{fig:VaryT+Dist}(right), we examine the relative errors for all the random pairs at a wide range of $\ell_2$ norms, for $t=100$ and $t=1000$. A slight change in the error profile occurs within $\|x_i-y_i\|/\sigma \in [10^0,10^1]$, coinciding with the separation of cases $\|x-y\| \leq \sigma$ and $\|x-y\| > \sigma$ i.e. whether $\frac{\|x-y\|}{\sigma} = \Theta(1)$ in the analyses.

In either case, the relative error is bounded by a small constant value, even when $\|x_i-y_i\|$ is several magnitudes smaller than $1$, demonstrating that the extremely high concentration of the RFF for very small $\|x_i-y_i\|$ results in relative error approximation for the Gaussian kernel distance.

\begin{figure}  
\includegraphics[width=0.45\textwidth]{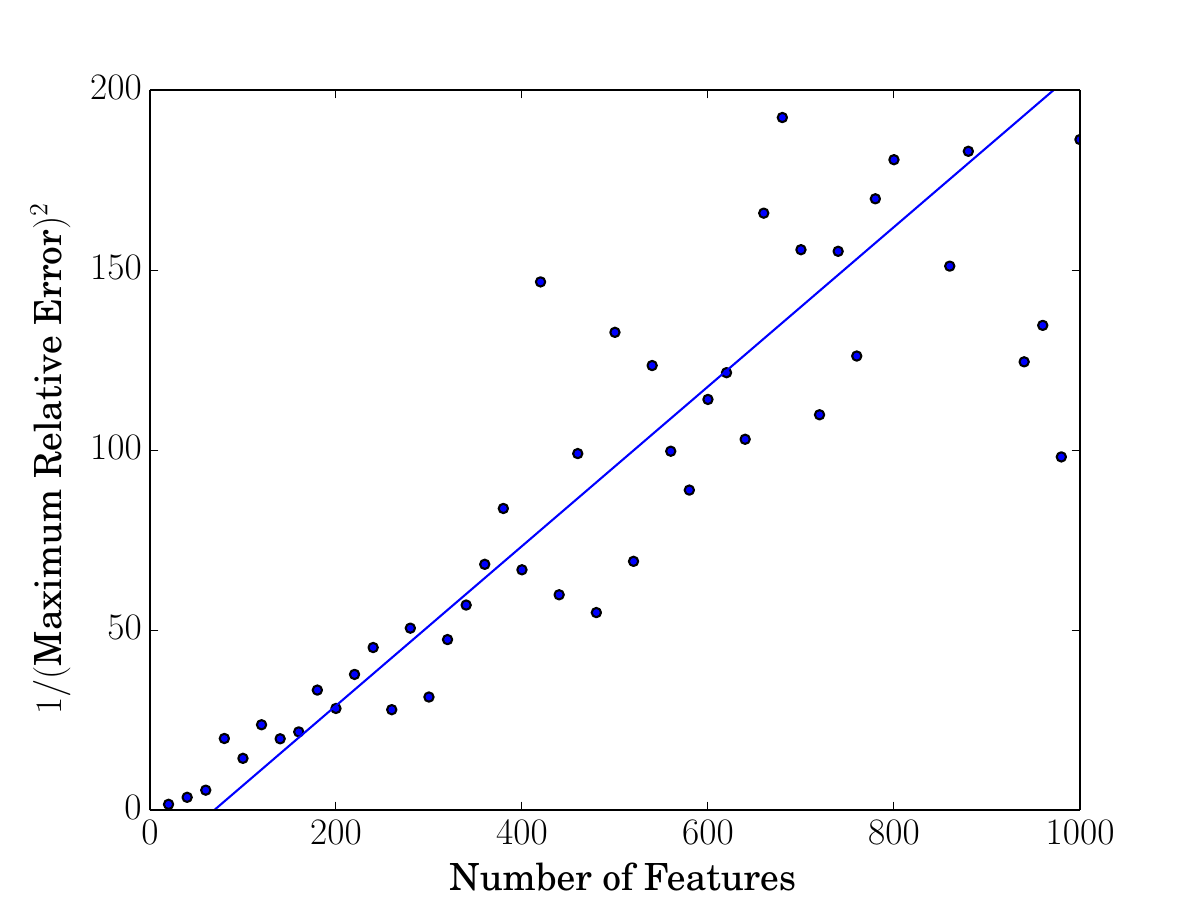}  
\hspace{0.09 \textwidth}
\includegraphics[width=0.45\textwidth]{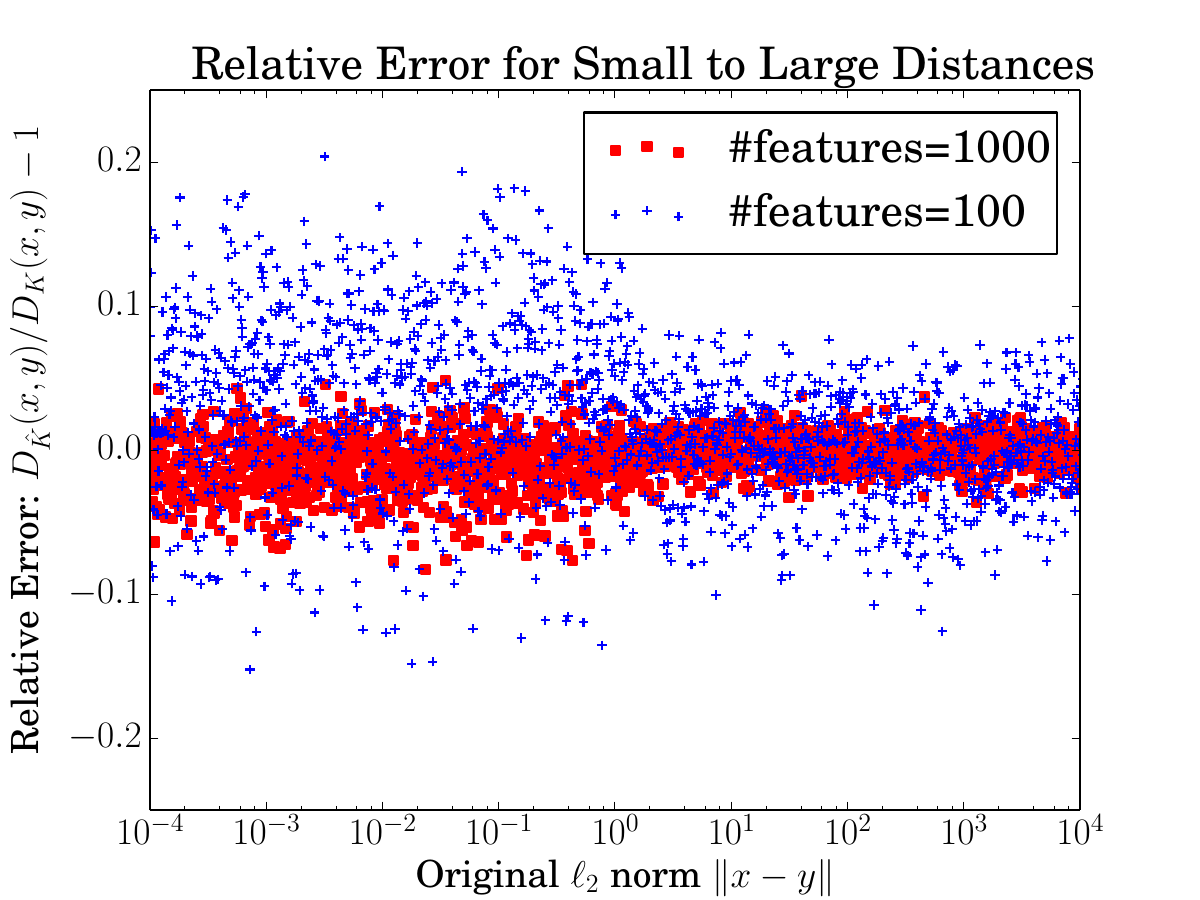}

\caption{\label{fig:VaryT+Dist}  \sffamily
(left) Inverse squared relative errors.  
(right) Relative errors with varying distance.} 
\end{figure}

\bibliographystyle{plain}
\bibliography{kernel}

\end{document}